\documentclass{article}

\usepackage{PRIMEarxiv}

\usepackage[utf8]{inputenc} 
\usepackage[T1]{fontenc}    
\usepackage{hyperref}       
\usepackage{url}            
\usepackage{booktabs}       
\usepackage{amsfonts}       
\usepackage{nicefrac}       
\usepackage{microtype}      
\usepackage{lipsum}
\usepackage{fancyhdr}       
\usepackage{graphicx}       
\graphicspath{{media/}}     

\usepackage{times}
\usepackage{latexsym}
\usepackage{amsthm}

\newtheorem{proposition}{Proposition}

\usepackage[T1]{fontenc}

\usepackage[utf8]{inputenc}

\usepackage{microtype}

\usepackage{inconsolata}
\usepackage{microtype}
\usepackage{graphicx}

\usepackage{hyperref}

\usepackage[textsize=tiny]{todonotes}
\usepackage{amssymb}

\usepackage{natbib} 
\usepackage{adjustbox}
\usepackage{graphicx}
\usepackage{caption}
\usepackage{subcaption}

\usepackage{amsmath}
\usepackage{bm}
\usepackage{multirow}
\usepackage{xcolor}
\usepackage[utf8]{inputenc}
\usepackage{mathtools}
\usepackage{bbm}
\usepackage{physics}
\usepackage{float}
\usepackage{booktabs}
 \usepackage{array, makecell}
 \usepackage{algorithm}
 \usepackage{algorithmic}
\usepackage{tabularx}
\usepackage{amsthm}

\usepackage{hyperref}
\title{Learning response-aware patient dynamics for respiratory support}

\author{
  Xiaolei Lu\thanks{Corresponding author. Email: luxiaolei0412@gmail.com},
  Shamim Nemati
}

\begin{document}

\maketitle

\begin{abstract}

Respiratory support can shape the short-term physiological trajectory of
critically ill patients, but patients receiving the same intervention may
follow different physiological trajectories. Clinical patient dynamics models typically predict future states from recent physiology and recorded
interventions, while physiological change is mainly represented through the predicted future state. We propose a response-aware patient dynamics model that explicitly represents physiological change during autoregressive state updating. The model decomposes predicted physiological change into state-dependent baseline dynamics and respiratory-support-associated deviations, with room air providing a reference for the decomposition. We provide a formal analysis of this reference-anchored formulation. A response pathway encodes the predicted physiological change and uses it to update the latent patient state across the forecast horizon. Across ICU cohorts from two independent institutions, the proposed model achieves comparable overall
trajectory prediction to patient dynamics baselines, with more consistent
improvements when physiological states are changing.

\end{abstract}

\section{Introduction}

Acute respiratory failure is a common reason for hospitalization and admission to the intensive care unit \citep{villgran2022acute}. Patients with hypoxemia are often first treated with conventional oxygen therapy. If respiratory failure worsens, support may be escalated to high-flow nasal cannula (HFNC), non-invasive ventilation (NIV), or invasive mechanical ventilation (IMV) \citep{azevedo2015high,coudroy2022high}. Physiological response to respiratory support varies across patients. Some patients improve after treatment, while others remain unstable or continue to deteriorate and may require further escalation. Changes in physiological measurements over short time intervals can provide useful information about the patient's evolving condition during respiratory support \citep{perez2025monitoring}. Modeling these short-term changes may therefore help characterize how physiology evolves under respiratory support. Figure~\ref{fp} illustrates this sequential setting and the role of physiological response in describing the transition from one patient state to the next.

Recent patient dynamics and clinical world models predict how patient states evolve over time under recorded interventions \citep{xu2026meddreamer,wu2026agentifying}. A common formulation represents the current patient state from recent clinical measurements and models the next state conditioned on the current intervention,
$p(s_{t+1}\mid s_t,a_t)$. The intervention may be represented by a treatment category, dose, or other recorded treatment variables \citep{yang2025medical,qazi2025beyond}. In these models, physiological progression and intervention-associated change are typically represented together within the same state-transition function.

\begin{figure}
    \centering
    \includegraphics[width=0.8\textwidth, height=0.3\textwidth]{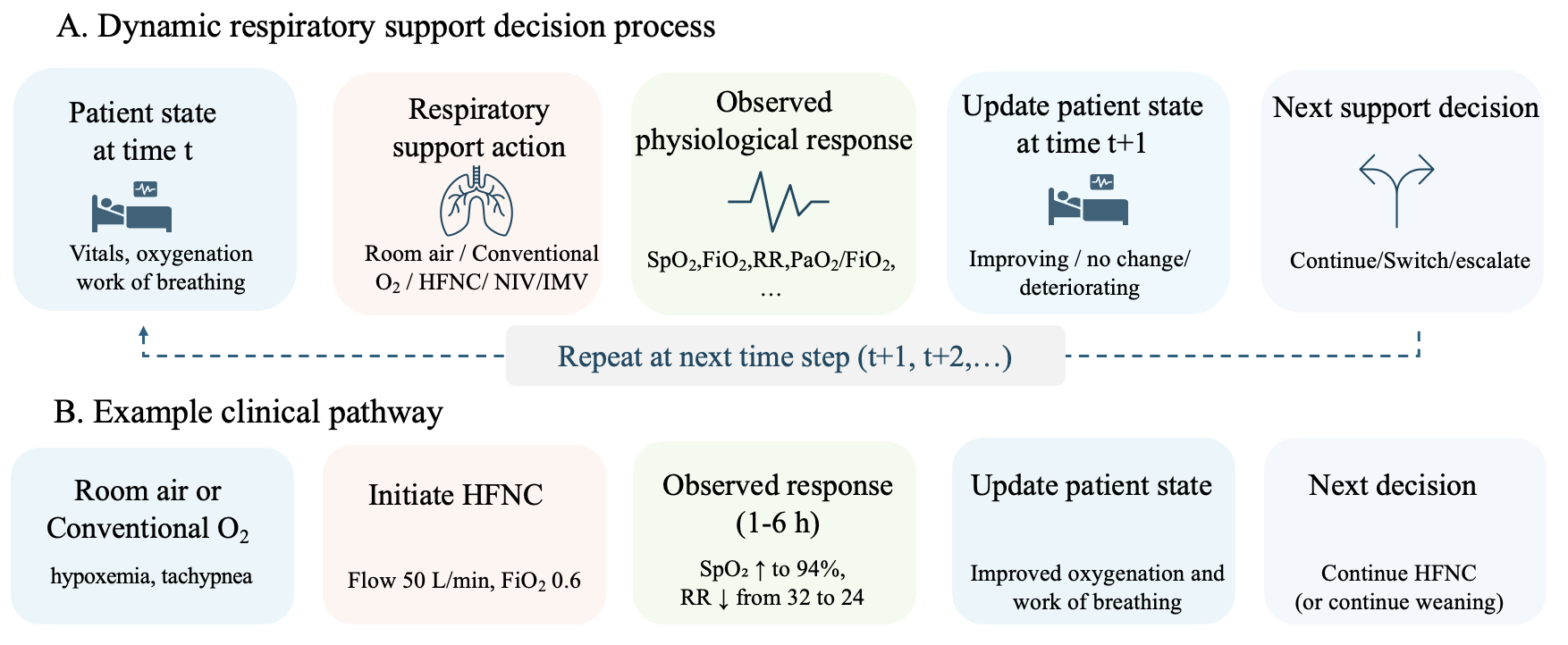}  
    \caption{Clinical motivation for response-aware respiratory dynamics.
(A) Respiratory support is delivered sequentially as patient physiology evolves. The respiratory support action space includes room air as a no-active-support category, since patients may experience periods without active respiratory support between support modalities in clinical practice. Following a respiratory support action, the observed physiological response provides new information that can update the patient state and inform the next support decision.
(B) Example clinical pathway showing how short-term response after HFNC initiation may help determine whether to continue the current support or consider escalation.}
    \label{fp}
\end{figure}

However, this formulation has two limitations for modeling respiratory support. First, physiological change after treatment may reflect both the patient's underlying progression and variation associated with the current respiratory support, but these components are usually modeled together. Second, action-conditioned dynamics models typically represent the next physiological level, without explicitly representing the change from the previous state during state updating. This change may provide additional information about how the patient's condition is evolving. Clinical studies of HFNC and NIV support this motivation. Patients who later succeed or fail treatment may have similar baseline measurements but show different physiological changes during the first hours of respiratory support \citep{mongodi2026early}, and the predictive value of response-based indices can increase after treatment initiation \citep{roca2019index}.

In this study, we extend action-conditioned patient dynamics modeling to
short-term respiratory-support forecasting with a response-aware formulation. The model separates state-dependent baseline dynamics from
respiratory-support-associated deviations and explicitly represents predicted physiological change during autoregressive state updating.

Our main contributions are as follows:
\begin{itemize}

\item We formulate short-term respiratory-support forecasting as an
action-conditioned patient dynamics problem in which physiological change is explicitly represented during trajectory prediction.

\item We develop a response-aware dynamics model that decomposes predicted
physiological change into state-dependent baseline dynamics and
respiratory-support-associated deviations, and uses the predicted change to update the latent patient state. We further analyze the decomposition using room air as a reference condition.

\item We evaluate the model on real-world ICU data from two institutions,
including temporal and external validation, and analyze performance across
forecast horizons, respiratory-support states, and clinically relevant
physiological variables.

\end{itemize}

\section{Related Work}

\subsection{Clinical Patient Dynamics and World Models}

Clinical trajectory models learn how patient states evolve over time from
longitudinal EHR data. Clinical world models further condition these dynamics on treatment actions to model future patient trajectories
\citep{xu2026meddreamer,wu2026agentifying,yang2025medical,qazi2025beyond}.
MedDreamer models patient dynamics in a latent state space using posterior
state inference from observed measurements and an action-conditioned prior for rollout \citep{xu2026meddreamer}. SepsisAgent uses an action-conditioned world model to predict short-term physiological states under different treatment actions \citep{wu2026agentifying}. These models differ in their state representations and transition architectures, but generally model physiological progression and intervention-associated change within the same transition process. Our work instead decomposes predicted physiological change into state-dependent baseline dynamics and respiratory-support-associated deviations, and explicitly represents predicted physiological change during state updating.

\subsection{Physiological Response to Respiratory Support}

Physiological response during respiratory support can provide information about a patient's evolving respiratory status. Following HFNC or NIV initiation, changes in oxygenation, respiratory rate, and related measurements have been used to assess subsequent treatment success or failure. For example, the predictive value of the ROX index increases during the hours following HFNC initiation \citep{roca2019index}, and other studies have reported different early physiological response patterns between patients who subsequently succeed or fail respiratory support \citep{mongodi2026early,perez2025monitoring}. These findings suggest that physiological change itself carries information about the patient's evolving state, beyond the physiological level at a single time point, which provides clinical motivation for explicitly representing physiological change in a patient dynamics model. Our model therefore encodes predicted physiological change and uses it during autoregressive state updating.

\section{Methodology}

\subsection{Problem formulation}

We denote the patient's recent physiological history over a look-back window of
length $L$ as $\bm{X}_{t-L+1:t}$, where
$\bm{x}_t \in \mathbb{R}^{p}$ contains the physiological measurements at time
$t$. At each time point, the respiratory intervention is represented by
$\bm{a}_t \in \mathbb{R}^{q}$, including the respiratory-support modality and
available modality-specific settings, such as oxygen flow and
$\mathrm{FiO}_2$ for oxygen support and airway pressures for non-invasive
ventilation. Static patient characteristics are denoted by $\bm{c}$.

Given the observed patient history and the recorded respiratory-support
sequence, our goal is to predict the physiological trajectory over the next
$H$ hours:
\begin{equation}
\hat{\bm{X}}_{t+1:t+H}=
f_\theta
\left(
\bm{X}_{t-L+1:t},
\bm{A}_{t-L+1:t+H-1},
\bm{c}
\right),
\end{equation}
where the model uses the observed respiratory-support history to initialize the patient state and is conditioned on the recorded respiratory-support action at each prediction step. Measurement-availability indicators and time-since-last-measurement variables are also included in the initial state representation.

We model the future trajectory autoregressively. At each prediction step, the predicted physiological change is decomposed into a state-dependent baseline component and a respiratory-support-associated deviation. The predicted change is then used both to obtain the next physiological state and to update the latent patient state before the following prediction step.

\subsection{Response-aware physiological dynamics}

We represent the patient using a latent state that is updated over the
prediction horizon. The model has four main components:
(1) initialization of the latent patient state from recent history,
(2) state-dependent encoding of respiratory support,
(3) decomposition of predicted physiological change into baseline dynamics and respiratory-support-associated deviations, and
(4) response-aware latent-state updating using the predicted physiological
change. An overview of the model architecture is provided in
Appendix~\ref{app:framework}.

\paragraph{Patient state initialization.}
A sequence encoder summarizes the recent patient trajectory using physiological measurements, time-since-last-measurement information, measurement-availability indicators, and respiratory-support history. Static patient characteristics and comorbidities are encoded separately.

The resulting dynamic and static representations are combined to initialize the latent patient state:
\begin{equation}
\bm{s}_t=
E_{\mathrm{state}}
\left(
\bm{X}_{t-L+1:t},
\bm{A}_{t-L+1:t},
\bm{T}_{t-L+1:t},
\bm{M}_{t-L+1:t},
\bm{c}
\right),
\label{eq:initial_state}
\end{equation}
where $\bm{T}$ denotes time-since-last-measurement information and
$\bm{M}$ denotes measurement-availability indicators. The initialized state $\bm{s}_t$ represents the patient's recent physiological trajectory,
respiratory-support history, and static characteristics at the start of
prediction.

\paragraph{State-dependent respiratory-support representation.}

Given the current latent patient state and respiratory intervention, we construct a state-dependent respiratory-support representation:
\begin{equation}
\bm{z}^{a}_t=
E_{\mathrm{action}}
\left(
\bm{s}_t,
\bm{a}_t
\right).
\label{eq:action}
\end{equation}

Conditioning the respiratory-support representation on the current latent state allows the same intervention to have different representations for different patient states.

\paragraph{Decomposition of physiological change.}

The predicted physiological change is decomposed into state-dependent baseline dynamics and a respiratory-support-associated deviation:
\begin{equation}
\Delta \hat{\bm{x}}_t=
F_{\mathrm{base}}(\bm{s}_t)
+
\bm{g}_t
\odot
F_{\mathrm{treat}}(\bm{z}^{a}_t),
\label{eq:dynamics}
\end{equation}
where
\begin{equation}
\bm{g}_t=
\sigma
\left(
G_{\mathrm{gate}}
\left(
\bm{s}_t,
\bm{a}_t
\right)
\right)
\label{eq:gate}
\end{equation}
is a feature-wise gate.

The baseline term $F_{\mathrm{base}}(\bm{s}_t)$ represents the component of predicted physiological change determined by the current latent patient state. The effective respiratory-support-associated deviation is $\bm{g}_t \odot F_{\mathrm{treat}}(\bm{z}^{a}_t)$ and represents additional variation modeled as a function of the recorded respiratory-support condition. Because other concurrent treatments are not explicitly separated in this formulation, the two components should be interpreted as a model-based decomposition rather than as isolated treatment effects.

We use room air as a reference condition for the decomposition. During
training, a soft regularization term encourages the
respiratory-support-associated deviation to be small for room-air steps.
Under the corresponding hard reference constraint, the additive decomposition is uniquely defined relative to the room-air reference; the formal result is provided in Appendix~\ref{app:decomposition_theory}. We use this decomposition to organize the predicted dynamics rather than to attribute physiological change to a causal treatment effect.

The next physiological state is obtained using a residual update:
\begin{equation}
\hat{\bm{x}}_{t+1}=
\tilde{\bm{x}}_t
+
\Delta\hat{\bm{x}}_t,
\label{eq:state_prediction}
\end{equation}
where $\tilde{\bm{x}}_t=\bm{x}_t$ at the first prediction step and
$\tilde{\bm{x}}_t=\hat{\bm{x}}_t$ thereafter.

\paragraph{Physiological response representation.}

After predicting the next physiological state, we define the short-term
physiological response as the change between consecutive states:
\begin{equation}
\hat{\bm{r}}_t=
\hat{\bm{x}}_{t+1}-
\tilde{\bm{x}}_t.
\label{eq:pred_response}
\end{equation}

The response is then encoded as
\begin{equation}
\bm{z}^{r}_t=
E_{\mathrm{response}}
\left(
\hat{\bm{r}}_t
\right),
\label{eq:response_encoder}
\end{equation}
where $E_{\mathrm{response}}$ is a learnable encoder used for latent-state
updating. During free-running prediction,
$\hat{\bm{r}}_t=\Delta\hat{\bm{x}}_t$. The response pathway therefore does not
introduce additional information, but provides a separate representation of the predicted change for the state-update module.

\paragraph{Response-aware latent-state updating.}

The latent patient state is updated using the predicted physiological state,
the state-dependent respiratory-support representation, and the encoded
physiological response:
\begin{equation}
\bm{s}_{t+1}=
U_{\mathrm{state}}
\left(
\bm{s}_t,
\hat{\bm{x}}_{t+1},
\bm{z}^{a}_t,
\bm{z}^{r}_t
\right),
\label{eq:update}
\end{equation}
where $U_{\mathrm{state}}$ denotes the state-update module.

The update  uses both the predicted physiological level and an
explicit representation of its change from the previous state. The resulting latent state is used to compute the baseline dynamics and
respiratory-support representation for the next prediction step:
$F_{\mathrm{base}}(\bm{s}_{t+1})$ and
$E_{\mathrm{action}}(\bm{s}_{t+1},\bm{a}_{t+1})$, respectively. Appendix~\ref{app:response_theory} provides additional analysis of the response pathway and its role in latent-state updating.

\subsection{Training}

\paragraph{Response-guided autoregressive training.}

The model is trained using multi-step autoregressive trajectory prediction. After initialization from the observed history, future physiological states are generated recursively:
\begin{equation}
\hat{\bm{x}}_{t+h}=
\hat{\bm{x}}_{t+h-1}
+
\Delta\hat{\bm{x}}_{t+h-1},
\qquad h=2,\ldots,H.
\label{eq:autoregressive}
\end{equation}

At each prediction step, the model-generated physiological state is carried forward to the next step; future ground-truth physiological states are not used as inputs to the trajectory rollout.

To stabilize training of the response pathway, we use a short response-guided schedule during the initial training epochs. For patient $i$ and prediction step $h$, the observed physiological response is
\begin{equation}
\bm{r}^{\mathrm{obs}}_{i,h}=\bm{x}_{i,t+h}-
\bm{x}_{i,t+h-1},
\label{eq:observed_response}
\end{equation}
and the model-generated response is
\begin{equation}
\bm{r}^{\mathrm{pred}}_{i,h}=\hat{\bm{x}}_{i,t+h}-
\tilde{\bm{x}}_{i,t+h-1},
\label{eq:model_response}
\end{equation}
where $\tilde{\bm{x}}_{i,t+h-1}$ denotes the state used as the reference for computing the response: it is the final observed physiological state at the first prediction step and the previous model-generated state at all subsequent steps.

During the response-guided period, the response supplied to the latent-state update is sampled independently for each patient and prediction step:
\begin{equation}
\bm{r}^{\mathrm{use}}_{i,h}=
b_{i,h}\bm{r}^{\mathrm{obs}}_{i,h}
+
\left(1-b_{i,h}\right)\bm{r}^{\mathrm{pred}}_{i,h},
\label{eq:response_teacher}
\end{equation}
where
\begin{equation}
b_{i,h}
\sim
\mathrm{Bernoulli}(\rho_e),
\end{equation}
and $\rho_e$ denotes the probability of using the observed response at epoch
$e$.

The probability $\rho_e$ is reduced to zero over the initial training period.
Observed responses are used only as inputs to the response pathway for
latent-state updating; the predicted physiological state is always carried
forward in the autoregressive rollout. After the response-guided period,
training uses only model-generated responses, matching the response input used
during validation and testing.

\paragraph{Respiratory-support regularization.}

Room air is used as the no-active-support reference condition. For prediction steps corresponding to room air, we penalize the magnitude of the effective respiratory-support-associated deviation:
\begin{equation}
\mathcal{L}_{\mathrm{RA}}=
\frac{1}{|\mathcal{I}_{\mathrm{RA}}|}
\sum_{(i,h)\in\mathcal{I}_{\mathrm{RA}}}
\left|
\bm{g}_{i,h}
\odot
F_{\mathrm{treat}}(\bm{z}^{a}_{i,h})
\right|_2^2,
\label{eq:ra_loss}
\end{equation}
where $\mathcal{I}_{\mathrm{RA}}$ denotes prediction steps corresponding to room air.

This regularization applies only to the respiratory-support-associated
deviation. It does not assume that physiology remains unchanged during
room-air periods; physiological change can still be represented by
$F_{\mathrm{base}}(\bm{s}_{i,h})$. The penalty therefore provides a soft
room-air reference for the decomposition. Under the corresponding hard
reference constraint, the additive decomposition is uniquely defined relative
to room air; the formal result is provided in
Appendix~\ref{app:decomposition_theory}.

\paragraph{Training objective.}

Let $m_{i,h,j}$ indicate whether physiological variable $j$ is observed for
patient $i$ at prediction horizon $h$. The trajectory-prediction loss is
\begin{equation}
\mathcal{L}_{\mathrm{pred}}
=
\frac{
\sum_{i=1}^{N}
\sum_{h=1}^{H}
\sum_{j=1}^{p}
m_{i,h,j}
\left(
x_{i,t+h,j}
-
\hat{x}_{i,t+h,j}
\right)^2
}{
\sum_{i=1}^{N}
\sum_{h=1}^{H}
\sum_{j=1}^{p}
m_{i,h,j}
},
\label{eq:prediction_loss}
\end{equation}
where the observation mask restricts the loss to available target measurements.

The complete training objective is
\begin{equation}
\mathcal{L}
=
\mathcal{L}_{\mathrm{pred}}
+
\lambda_{\mathrm{RA}}
\mathcal{L}_{\mathrm{RA}},
\label{eq:total_loss}
\end{equation}
where $\lambda_{\mathrm{RA}}$ controls the strength of the room-air
regularization.

\paragraph{Free-running prediction.}
During validation, testing, and external evaluation, no future physiological measurements or observed responses are provided. Predicted states are recursively
carried forward, and the model-generated response satisfies
$\hat{\bm r}_t=\Delta\hat{\bm x}_t$. Thus, after initialization, prediction depends only on model-generated states and the specified respiratory-support sequence.

\section{Experiments}

\subsection{Experimental setting}
\noindent{\textbf{Datasets}}.
We conducted a retrospective study using de-identified EHR data from adult ICU patients ($\geq$18 years). Data from January 1, 2016, to December 31, 2023, were divided into a development cohort (\textit{Dev}) for model training and an internal validation cohort (\textit{Val}) for model selection. We additionally evaluated the model on a later-period cohort from the same institution (\textit{Temporal}; January 1, 2024--June 30, 2024) and an independent cohort from a separate institution (\textit{External}; January 1, 2023--August 31, 2024). All analyses were conducted under institutional review board approval with waiver of informed consent. Cohort characteristics are summarized in Appendix Table~\ref{tab:cohort_characteristics}, with cohort construction and
preprocessing details provided in Appendix~\ref{data}.

\noindent{\textbf{Implementation Details}}.
We optimize the model using AdamW \citep{loshchilov2017decoupled}. Model
selection is based on the lowest validation RMSE, with early stopping using a patience of 10 epochs. We use a four-epoch response-guided warm-up, with $\rho_e$ linearly decreased from 1 to 0 across the four epochs. After the warm-up period, $\rho_e=0$ and training uses only model-generated responses. Hyperparameters are tuned on the internal validation cohort using Optuna \citep{akiba2019optuna}. We search over learning rate, latent-state dimension, action-representation dimension, network depth, dropout, weight decay, and $\lambda_{\mathrm{RA}}$. All learning-based baselines use the same training and validation splits,
preprocessing pipeline, physiological variables, respiratory-support inputs, prediction horizon, and model-selection objective. Hyperparameters are tuned independently for each baseline using the same validation objective and a comparable search space and number of trials.

\noindent{\textbf{Evaluation}}.
We evaluate multi-step physiological trajectory prediction in a fully
free-running setting. After initialization from a 6-hour observed history, the model recursively predicts the following 6 hours of physiology without access to future ground-truth physiological states. We use a 6-hour prediction horizon to focus on short-term physiological changes during respiratory support, which can be clinically informative for assessing evolving respiratory status and deterioration \citep{roca2019index}. Performance is evaluated only at observed target measurements.

Our primary metric is root mean squared error (RMSE), reported over the full 1--6 h prediction horizon and separately at each forecast hour. Overall RMSE is computed by pooling all observed feature-level targets across the six forecast hours and taking the square root of the mean squared error. We also report mean absolute error (MAE). Because persistence can perform strongly when physiology changes little over
short intervals, we separately evaluate targets with non-trivial physiological changes. For each feature and forecast horizon, a target is included when $|x_{t+h}-x_t|>0.25$ SD and the ground-truth value is observed, where SD is the training-set standard deviation of that variable. This criterion removes near-stable targets while accounting for differences in scale across physiological variables. We report MAE and RMSE on these targets, together with the relative RMSE improvement over persistence. We further stratify performance by respiratory-support modality, including room air, conventional oxygen therapy, HFNC, and NIV.

\subsection{Ablation and comparison with baselines}

\paragraph{Ablation study.}
Table~\ref{tab:ablation} evaluates the effects of respiratory-support-associated dynamics and the response-aware state update. On the validation cohort, adding respiratory-support-associated dynamics reduces MSE from 0.270 to 0.268 and MC-MAE from 0.736 to 0.726, while the full model further reduces MSE to 0.266. On the temporal and external cohorts, respiratory-support-associated dynamics alone do not consistently improve performance, whereas the response-aware update recovers or improves several metrics, including MC-MAE. These results suggest that respiratory-support conditioning alone is insufficient, and that feeding predicted physiological change back into the latent state improves the stability of the learned dynamics across cohorts.

\begin{table}[t]
\centering
\caption{Ablation study of the proposed patient dynamics model across three evaluation cohorts.}
\label{tab:ablation}

\small
\setlength{\tabcolsep}{3pt}
\renewcommand{\arraystretch}{1.05}


\begin{tabular}{@{}lccccc@{}}
\hline

\multicolumn{6}{c}{\textbf{Val}} \\
\hline
\textbf{Model}
& \textbf{MSE $\downarrow$}
& \textbf{1-h RMSE $\downarrow$}
& \textbf{3-h RMSE $\downarrow$}
& \textbf{6-h RMSE $\downarrow$}
& \textbf{MC-MAE $\downarrow$} \\
\hline

Base dynamics
& 0.270
& 0.369
& 0.504
& 0.621
& 0.736 \\

+ Treatment-associated dynamics
& 0.268
& \textbf{0.361}
& 0.506
& \textbf{0.619}
& \textbf{0.710} \\

+ Response-aware state update (Full)
& \textbf{0.266}
& \textbf{0.361}
& 0.505
& \textbf{0.619}
& 0.726 \\

\hline

\multicolumn{6}{c}{\textbf{Temporal}} \\
\hline
\textbf{Model}
& \textbf{MSE $\downarrow$}
& \textbf{1-h RMSE $\downarrow$}
& \textbf{3-h RMSE $\downarrow$}
& \textbf{6-h RMSE $\downarrow$}
& \textbf{MC-MAE $\downarrow$} \\
\hline

Base dynamics
& 0.304
& 0.398
& 0.545
& \textbf{0.646}
& 0.728 \\

+ Treatment-associated dynamics
& 0.304
& 0.396
& 0.538
& 0.653
& 0.772 \\

+ Response-aware state update (Full)
& \textbf{0.302}
& \textbf{0.389}
& 0.539
& 0.653
& \textbf{0.710} \\

\hline

\multicolumn{6}{c}{\textbf{External}} \\
\hline
\textbf{Model}
& \textbf{MSE $\downarrow$}
& \textbf{1-h RMSE $\downarrow$}
& \textbf{3-h RMSE $\downarrow$}
& \textbf{6-h RMSE $\downarrow$}
& \textbf{MC-MAE $\downarrow$} \\
\hline

Base dynamics
& 0.778
& 0.624
& \textbf{0.819}
& \textbf{1.061}
& 0.911 \\

+ Treatment-associated dynamics
& 0.814
& 0.625
& 0.835
& 1.103
& \textbf{0.904}\\

+ Response-aware state update (Full)
& \textbf{0.775}
& \textbf{0.620}
& \textbf{0.819}
& 1.062
& 0.906\\

\hline
\end{tabular}



\vspace{2pt}

\begin{minipage}{0.98\linewidth}
\footnotesize
MC-MAE denotes mean absolute error on physiologically meaningful-change
targets, defined as $|x_{t+h}-x_t| > 0.25$ SD.
Lower values indicate better performance.
\end{minipage}

\end{table}

\noindent{\textbf{Baseline models}}.
We compare against representative physiological trajectory and clinical
world-model baselines: Persistence, Action-GRU, SepsisAgent-WM
\citep{wu2026agentifying}, and MedDreamer-WM
\citep{xu2026meddreamer}. Persistence carries the most recent observed
physiological state forward across the prediction horizon. Action-GRU is a
standard action-conditioned recurrent model for autoregressive physiological forecasting. SepsisAgent-WM and MedDreamer-WM are adaptations of recent clinical world models to our respiratory trajectory prediction task. Because both literature-derived dynamics models use GRU-based transition architectures, we use a GRU-based action-conditioned model as the standard recurrent baseline to provide a more direct comparison of the dynamics formulations. For the literature-derived models, we preserve their core transition architectures and modify only the task-specific input and output interfaces required for our setting. Additional baseline details are provided in Appendix~\ref{app:baselines}.

\paragraph{Comparison with baselines.}
Table~\ref{tab:clinical_world_models} shows that the learning-based models have similar overall RMSE, but their performance differs across forecast horizons. On the validation cohort, MedDreamer-WM has lower RMSE during the first 3 hours, whereas our model has lower RMSE from 4 to 6 hours. On the temporal cohort, MedDreamer-WM and our model have the same overall RMSE of 0.550; MedDreamer-WM performs better at the earliest horizons, while our model has lower RMSE at 3--5 h. On the external cohort, our model achieves the lowest
RMSE consistently from 2 to 5 h, although SepsisAgent-WM has the lowest overall RMSE.

These results suggest different strengths across models. MedDreamer-WM is stronger at the earliest forecast steps, whereas our model is more competitive over intermediate and later horizons, when predictions depend increasingly on
repeated state updates. Together with the ablation results, this pattern is consistent with the proposed formulation: decomposing physiological change and feeding the predicted change back into the latent state appears most useful
during multi-step autoregressive rollout, rather than at the first prediction step.

\paragraph{Meaningful physiological changes.}
We further evaluate targets with
$|x_{t+h}-x_t|>0.25$ SD (Table~\ref{tab:meaningful_change}).
On the validation cohort, our model achieves an RMSE of 1.016, corresponding to a 10.1\% improvement over persistence, compared with 9.2\% for MedDreamer-WM. On the temporal cohort, the corresponding improvements are 9.7\% and 9.6\%, respectively. On the external cohort, our model and Action-GRU both achieve an RMSE of 1.771, corresponding to a 4.2\% improvement over persistence. MedDreamer-WM has the lowest external-cohort MAE (0.897), while our model achieves 0.906.

The larger gains over persistence on these changing targets show that the differences between models become more apparent when physiology is evolving. Our model remains among the best-performing methods across all three cohorts, which is consistent with the intended role of explicitly representing
physiological change and feeding the predicted change back into the latent state during rollout.

\begin{table*}[t]
\centering
\caption{
Comparison with clinical world models for multi-step physiological trajectory prediction (values are RMSE; lower is better).
}
\label{tab:clinical_world_models}

\small
\setlength{\tabcolsep}{4.8pt}
\renewcommand{\arraystretch}{1.0}

\begin{tabular}{lccccccc}
\hline
\textbf{Model}
& \textbf{Overall}
& \textbf{1h}
& \textbf{2h}
& \textbf{3h}
& \textbf{4h}
& \textbf{5h}
& \textbf{6h} \\
\hline

\multicolumn{8}{l}{\textit{Val}} \\

Persistence
& 0.549
& 0.370
& 0.469
& 0.535
& 0.591
& 0.619
& 0.659 \\

Action-GRU

& 0.520
& 0.355
& 0.445
& 0.507
& 0.559
& 0.585
& 0.624 \\

SepsisAgent-WM
&0.519&0.355 & 0.445 & 0.508 & 0.559 & 0.585&0.624\\

MedDreamer-WM
& \textbf{0.516}
& \textbf{0.344}
& \textbf{0.435}
& \textbf{0.502}
& 0.556
& 0.584
& 0.624 \\

\textbf{Ours}
& \textbf{0.516}
& 0.361
& 0.440
& 0.505
& \textbf{0.550}
& \textbf{0.578}
& \textbf{0.619} \\

\multicolumn{8}{l}{\textit{Temporal}} \\

Persistence
& 0.587
& 0.400
& 0.515
& 0.583
& 0.627
& 0.658
& 0.694 \\

Action-GRU
& 0.552
& 0.383
& 0.492
& 0.546
& 0.585
& 0.617
& \textbf{0.652} \\

SepsisAgent-WM
& 0.552
& 0.383
& 0.495
& 0.547
& 0.585
& \textbf{0.615}
& 0.649 \\

MedDreamer-WM
& \textbf{0.550}
& \textbf{0.369}
& \textbf{0.485}
& 0.543
& 0.584
& 0.618
& 0.654 \\

\textbf{Ours}
& \textbf{0.550}
& 0.389
& 0.487
& \textbf{0.539}
& \textbf{0.579}
& \textbf{0.615}
& 0.653 \\

\multicolumn{8}{l}{\textit{External}} \\

Persistence
& 0.901
& 0.624
& 0.729
& 0.840
& 0.979
& 1.057
& 1.084 \\

Action-GRU
& 0.880
& \textbf{0.613}
& 0.707
& 0.822
& 0.958
& 1.028
& 1.061 \\

SepsisAgent-WM
& \textbf{0.878}
& \textbf{0.613}
& 0.708
& 0.821
& 0.959
& 1.028
& \textbf{1.055}\\

MedDreamer-WM
& 0.889
& 0.610
& 0.708
& 0.829
& 0.970
& 1.044
& 1.076 \\

\textbf{Ours}
& 0.880
& 0.620
& \textbf{0.706}
& \textbf{0.819}
& \textbf{0.957}
& \textbf{1.027}
& 1.062 \\

\hline

\end{tabular}

\end{table*}

\begin{table*}[t]
\centering
\caption{
Performance on physiologically meaningful changes
($|x_{t+h}-x_t|>0.25$ SD).
$\Delta$ denotes relative RMSE improvement over persistence.
}
\label{tab:meaningful_change}

\small
\setlength{\tabcolsep}{4pt}
\renewcommand{\arraystretch}{1.02}

\begin{tabular}{lccc|ccc|ccc}
\hline
& \multicolumn{3}{c|}{\textbf{Val}}
& \multicolumn{3}{c|}{\textbf{Temporal}}
& \multicolumn{3}{c}{\textbf{External}} \\
\textbf{Model}
& \textbf{MAE} & \textbf{RMSE} & \textbf{$\Delta$ (\%)}
& \textbf{MAE} & \textbf{RMSE} & \textbf{$\Delta$ (\%)}
& \textbf{MAE} & \textbf{RMSE} & \textbf{$\Delta$ (\%)} \\
\hline

Persistence
& 0.831 & 1.131 & --
& 0.828 & 1.224 & --
& 1.002 & 1.848 & -- \\

Action-GRU
& 0.738 & 1.036 & +8.4
& 0.733 & 1.118 & +8.7
& 0.911 & \textbf{1.771} & \textbf{+4.2} \\

SepsisAgent-WM
& 0.742 & 1.041 & +7.9
& 0.735 & 1.124 & +8.2
& 0.911 & 1.777 & +3.8 \\

MedDreamer-WM
& \textbf{0.726} & 1.027 & +9.2
& 0.718 & 1.106 & +9.6
& \textbf{0.897} & 1.783 & +3.5 \\

\textbf{Ours}
& \textbf{0.726} & \textbf{1.016} & \textbf{+10.1}
& \textbf{0.710} & \textbf{1.105} & \textbf{+9.7}
& 0.906 & \textbf{1.771} & \textbf{+4.2} \\

\hline
\end{tabular}

\vspace{2pt}
\begin{minipage}{0.98\linewidth}
\footnotesize
\end{minipage}

\end{table*}
\subsection{Clinical Stratified Analysis }

\paragraph{Performance across respiratory-support states.}
We stratify performance by respiratory-support state to assess whether the overall results are consistent across support conditions or driven by a single subgroup. Figure~\ref{fig:support_stratified} shows the relative RMSE improvement over persistence for room air, conventional oxygen, HFNC, and NIV. Our model remains competitive across all support states and achieves the lowest RMSE in several subgroups, including conventional oxygen and NIV on the validation cohort, room air and HFNC on the temporal cohort, and conventional oxygen and NIV on the external cohort. Differences between learning-based models are generally small, particularly for HFNC and the smaller NIV subgroups. Overall, these results suggest that the model's performance is not driven by a single respiratory-support category. Full subgroup results are provided in Appendix~\ref{app:support_stratification}.

\begin{figure*}[t]
    \centering
    \includegraphics[width=0.95\textwidth]{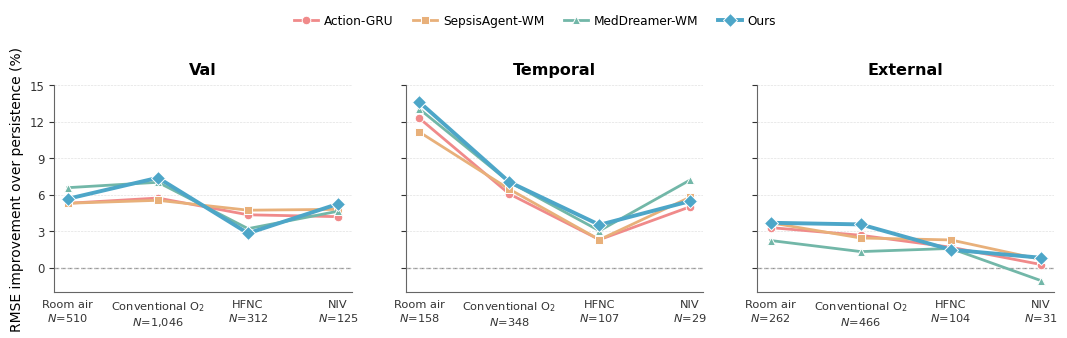}
    \caption{
    Relative RMSE improvement over persistence across respiratory-support
    modalities in the validation, temporal, and external cohorts.
    Positive values indicate lower RMSE than persistence.
    $N$ denotes the number of evaluation windows in each subgroup.
    }
    \label{fig:support_stratified}
\end{figure*}

\paragraph{Performance across physiological variables.}
We further analyze performance at the feature level to determine whether the improvements on changing physiological states are driven by a small number of variables or are distributed across different aspects of patient physiology. We focus on selected variables related to respiratory status and hemodynamics, including oxygen saturation (O$_2$Sat), respiratory rate (Resp), fraction of inspired oxygen (FiO$_2$), arterial oxygen and carbon dioxide partial pressures (PaO$_2$ and PaCO$_2$), pH, heart rate (HR), mean arterial pressure (MAP), and systolic and diastolic blood pressure (SBP and DBP). For each feature, evaluation is restricted to targets satisfying $|x_{t+h}-x_t|>0.25$ SD.

Figure~\ref{fig:support_meaningful} shows that the improvement over persistence is distributed across multiple variables rather than being driven by a single measurement. Our model achieves the lowest RMSE for 6 of 10 variables on the validation cohort, 5 of 10 on the temporal cohort, with a tie for HR, and 8 of 10 on the external cohort. The largest gains are observed for several respiratory variables, particularly FiO$_2$, while improvements are also seen for hemodynamic variables such as MAP and DBP. These results suggest that the benefit of the response-aware formulation is not limited to a single physiological variable or measurement type. Full absolute RMSE values and the number of evaluated targets for each feature are reported in Appendix
Table~\ref{tab:feature_meaningful}.

\begin{figure*}[t]
    \centering
    \includegraphics[width=0.95\textwidth]{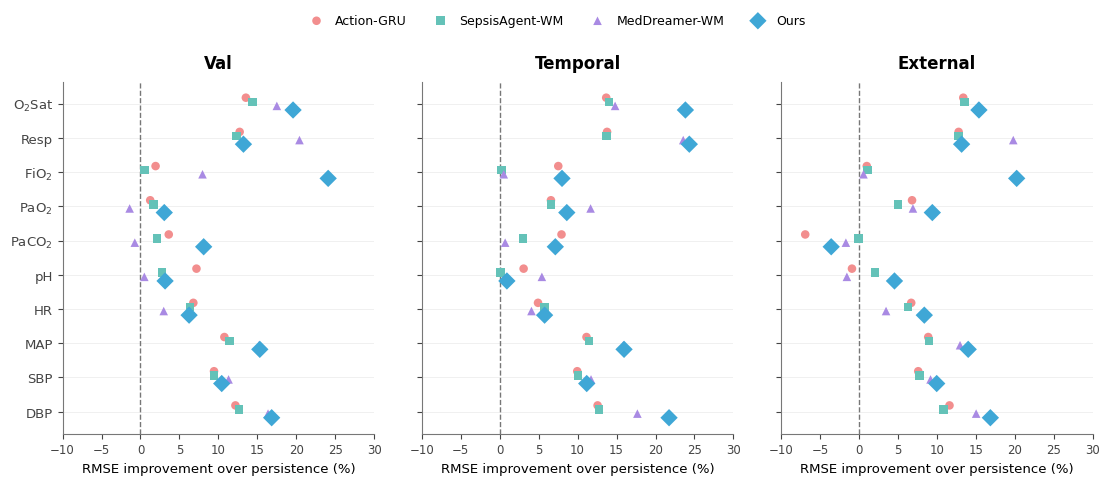}
    \caption{
  Relative RMSE improvement over persistence for ten physiological variables in the validation, temporal, and external cohorts, evaluated on targets exhibiting meaningful changes ($|x_-x_t|>0.25$ SD). Positive values indicate lower RMSE than persistence, with larger values reflecting greater improvements.
    }
    \label{fig:support_meaningful}
\end{figure*}

\section{Conclusion}

We introduced a response-aware patient dynamics model for short-term
physiological forecasting under respiratory support. The model decomposes
predicted physiological change into state-dependent baseline dynamics and
respiratory-support-associated deviations, and explicitly represents the
predicted change during latent-state updating. Across internal, temporal, and external cohorts, the model achieved overall performance comparable to strong patient dynamics baselines, with more consistent improvements when physiological states were changing.

The current study is retrospective and focuses on observational trajectory
prediction rather than causal treatment-effect estimation or treatment
recommendation. Future work will investigate whether the learned dynamics can support counterfactual trajectory analysis and respiratory-support decision modeling, with explicit treatment of confounding and prospective validation.

\section*{AI Use Disclosure}

Generative AI tools were used to assist with manuscript drafting and editing. All AI-assisted text is independently reviewed and verified by the authors. The authors take full
responsibility for the final content of the paper.

\section*{Ethics Statement}

This study uses retrospective, de-identified electronic health record data from multiple clinical institutions under applicable institutional review, data-use, and data-governance requirements. Patient consent was waived where permitted for retrospective analysis of de-identified data. All analyses were conducted in
accordance with institutional privacy and security policies, and no attempt was made to re-identify individual patients. The underlying clinical data cannot be publicly released where restricted by institutional policies or data-use agreements.

The proposed model is intended solely as a research framework for modeling short-term physiological trajectories under respiratory support and is not designed for autonomous clinical decision-making. Model predictions should not be interpreted as guaranteed patient responses or as recommendations for respiratory-support selection. Any prospective clinical use would require additional external and prospective validation, assessment of failure modes, clinical oversight, and appropriate regulatory review.

Clinical practice patterns, patient populations, measurement processes, and respiratory-support use may differ across institutions. Although we evaluate the model across multiple cohorts to assess robustness under distribution shift, residual bias and under-representation may remain, and performance may not
generalize to all patient groups or care settings.

Finally, the decomposition of predicted physiological change into underlying progression and respiratory-support-associated components is a modeling construct and should not be interpreted as identifying a causal treatment effect. Respiratory interventions are observed rather than randomized and may be influenced by measured and unmeasured clinical factors. Counterfactual claims regarding the effectiveness of alternative respiratory-support
strategies are therefore outside the scope of this work.

\section*{Reproducibility Statement}

We provide detailed descriptions of cohort construction, trajectory generation, preprocessing, respiratory-support representation, model architecture, training, and evaluation in the main text and appendix. Each modeling example consists of a fixed observed history followed by a fixed prediction horizon, and all reported
trajectory results are obtained using free-running autoregressive prediction without access to future ground-truth physiology.

The architecture specification includes latent-state initialization,
state-dependent respiratory-support encoding, decomposition of physiological change into underlying and treatment-associated components, feature-wise treatment gating, recurrent latent-state updating, and all training objectives and regularization terms. We also document the optimization settings, hyperparameters, horizon weighting scheme, model-selection procedure, and random seeds used in the reported experiments.

We evaluate the proposed method using predefined trajectory-level and
horizon-specific error metrics across internal and external cohorts, together with baseline comparisons, ablation studies, and sensitivity analyses. Subject to institutional and data-use restrictions, we will release the implementation
of preprocessing, trajectory construction, model training, free-running evaluation, baseline models, and scripts used to generate the reported tables and figures, together with configuration files sufficient to reproduce the experiments on appropriately formatted data. The underlying clinical datasets cannot be publicly released where restricted by institutional policies or data-use agreements.

\bibliographystyle{unsrtnat}
\bibliography{reference}  

\appendix

\section{Model architecture overview}
\label{app:framework}

Figure~\ref{fig:framework_overview} summarizes the response-aware patient dynamics framework and the information flow across the autoregressive prediction horizon.

\begin{figure*}[t]
    \centering
    \includegraphics[width=1\textwidth]{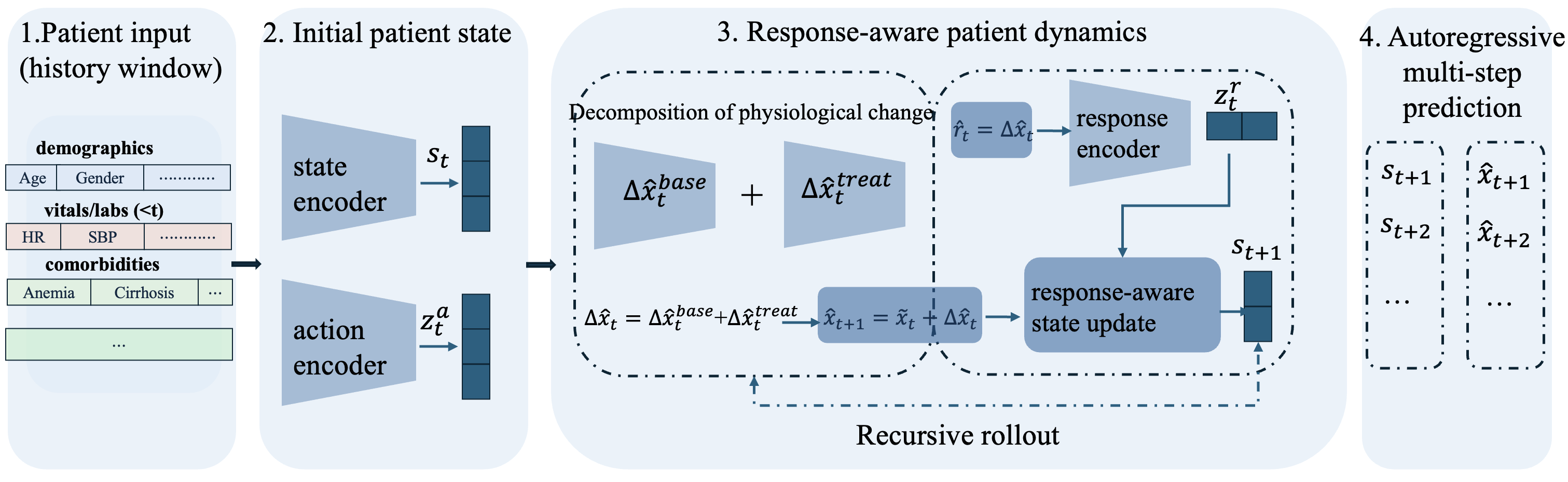}
    \caption{
   Overview of the response-aware patient dynamics model. The model initializes a latent patient state from recent clinical history and predicts physiological change as the sum of state-dependent baseline dynamics and a respiratory-support-associated deviation. The predicted change produces the next physiological state through a residual update and is also explicitly encoded as the physiological response for latent-state updating. The resulting state is recursively propagated over the prediction horizon.
    }
    \label{fig:framework_overview}
\end{figure*}

\section{Data preprocessing and trajectory construction}

\label{data}

Table~\ref{tab:cohort_characteristics} summarizes cohort demographics,
prediction-window counts, target availability, and respiratory-support
distributions across the development, internal validation, temporal, and
external cohorts.

\begin{table*}[!t]
\centering
\caption{Characteristics of the study cohorts.}
\label{tab:cohort_characteristics}
\footnotesize
\setlength{\tabcolsep}{6pt}
\renewcommand{\arraystretch}{0.98}

\begin{tabular}{lrrrr}
\toprule
\textbf{Characteristic}
& \textbf{Dev}
& \textbf{Val}
& \textbf{Temporal}
& \textbf{External} \\
\midrule

\multicolumn{5}{l}{\textit{Cohort size}} \\
Unique patients, $N$
& 6,762 & 1,912 & 564 & 791 \\

\addlinespace[2pt]
\multicolumn{5}{l}{\textit{Demographics (unique patient level)}} \\
Age, years, mean $\pm$ SD
& 62.9 $\pm$ 16.5
& 62.1 $\pm$ 16.6
& 64.3 $\pm$ 16.4
& 64.5 $\pm$ 17.2 \\
Female, $N$ (\%)
& 2,835 (41.9\%)
& 795 (41.6\%)
& 247 (43.8\%)
& 350 (44.2\%) \\
\quad White
& 3,524 (52.1\%)
& 933 (48.8\%)
& 264 (46.8\%)
& 385 (48.7\%) \\
\quad Black or African American
& 514 (7.6\%)
& 142 (7.4\%)
& 49 (8.7\%)
& 18 (2.3\%) \\
\quad Asian
& 509 (7.5\%)
& 144 (7.5\%)
& 49 (8.7\%)
& 176 (22.3\%) \\
\quad Other/Unknown
& 2,215 (32.8\%)
& 693 (36.2\%)
& 202 (35.8\%)
& 212 (26.8\%) \\

\addlinespace[2pt]
\multicolumn{5}{l}{\textit{Prediction data}} \\
Prediction windows, $N$
& 7,972 & 1,993 & 642 & 863 \\
Available target observations, $N$
& 1,046,277 & 264,366 & 85,590 & 116,805 \\
Meaningful-change targets, $N$ (\%)
& 27,981 (2.7\%)
& 7,316 (2.8\%)
& 2,073 (2.4\%)
& 2,173 (1.9\%) \\

\addlinespace[2pt]
\multicolumn{5}{l}{\textit{Respiratory support at prediction-window start}} \\
\quad Room air
& 1,955 (24.5\%)
& 510 (25.6\%)
& 158 (24.6\%)
& 262 (30.4\%) \\
\quad Conventional O$_2$
& 4,362 (54.7\%)
& 1,046 (52.5\%)
& 348 (54.2\%)
& 466 (54.0\%) \\
\quad HFNC
& 1,188 (14.9\%)
& 312 (15.7\%)
& 107 (16.7\%)
& 104 (12.1\%) \\
\quad NIV
& 467 (5.9\%)
& 125 (6.3\%)
& 29 (4.5\%)
& 31 (3.6\%) \\
\bottomrule
\end{tabular}

\vspace{2pt}
\begin{minipage}{0.98\textwidth}
\footnotesize
\textit{Note:}
Demographic characteristics are summarized at the unique
patient level. Respiratory-support percentages are calculated over
prediction windows, and meaningful-change percentages
are calculated over available target observations.
\end{minipage}
\end{table*}

\paragraph{Trajectory representation.}
Each patient encounter was discretized into hourly bins and represented as a longitudinal trajectory containing physiological measurements,
respiratory-support information, time-since-last-measurement variables, and static patient characteristics. Trajectories extended from ICU admission until transfer out of the ICU or initiation of invasive mechanical ventilation (IMV), whichever occurred first. Respiratory-support actions before IMV included room air, conventional oxygen therapy, high-flow nasal cannula (HFNC), and non-invasive ventilation (NIV).

At each hour $t$, the physiological state $\bm{x}_t \in \mathbb{R}^{50}$ contained 50 dynamic clinical variables. Each variable was paired with a corresponding time-since-last-measurement feature, forming a 50-dimensional TSLM vector. Static patient information consisted of
36 demographic and clinical variables and 62 comorbidity indicators, which
were fixed for each trajectory and remained unchanged during autoregressive prediction.

\paragraph{Physiological feature preprocessing.}
Missing values in the 50 physiological variables were handled using
feature-wise median imputation. To avoid information leakage, all imputation
parameters were estimated using the training cohort only. For 
feature $j$, let $m_j^{\mathrm{tr}}$ denote the median of the observed training
values. A missing observation was replaced by

\begin{equation}
\tilde{x}_{t,j}
=
\begin{cases}
x_{t,j},
&
x_{t,j}\ \text{observed},
\\[4pt]
m_j^{\mathrm{tr}},
&
x_{t,j}\ \text{missing}.
\end{cases}
\end{equation}

After imputation, each physiological variable was standardized using the mean
and standard deviation estimated from the imputed training data. Let
$\mu_j^{\mathrm{tr}}$ and $\sigma_j^{\mathrm{tr}}$ denote these statistics.
The normalized physiological variable was computed as

\begin{equation}
x_{t,j}^{\mathrm{norm}}
=
\frac{
\tilde{x}_{t,j}
-
\mu_j^{\mathrm{tr}}
}{
\sigma_j^{\mathrm{tr}}
}.
\end{equation}

The same training-derived median, mean, and standard deviation were applied without re-estimation to the validation and external cohorts. We additionally retained a binary missingness indicator for each physiological variable, denoting whether the corresponding value had originally been observed or imputed. The time-since-last-measurement variables were retained separately from the normalized physiological measurements and were provided to the initial state encoder.

\paragraph{Respiratory intervention representation.}
The respiratory support administered at each time point was represented by a
treatment vector $\bm{a}_t$ containing the respiratory support modality,
device subtype, available device settings, and indicators describing whether
individual settings were observed. The support modality and device subtype
were categorical variables and were encoded using one-hot vocabularies
constructed from the training cohort.

Continuous respiratory support settings consisted of
$\mathrm{FiO}_2$, oxygen flow, IPAP, and EPAP. These settings are not
applicable to every respiratory support modality. For example, flow may be
available for high-flow oxygen support, whereas IPAP and EPAP are primarily
associated with non-invasive positive-pressure ventilation. We therefore
paired each continuous setting with a binary availability mask. For a
treatment setting $q_t$ with corresponding mask $r_t$, only training
observations satisfying $r_t=1$ were used to estimate its normalization
statistics. When the setting was available, it was normalized as

\begin{equation}
q_t^{\mathrm{norm}}
=
\frac{
q_t-\mu_q^{\mathrm{tr}}
}{
\sigma_q^{\mathrm{tr}}
},
\qquad r_t=1,
\end{equation}

where $\mu_q^{\mathrm{tr}}$ and $\sigma_q^{\mathrm{tr}}$ were estimated from valid training observations. When $r_t=0$, the normalized setting was assigned a value of zero while the availability mask remained zero, which allows the model to distinguish an unavailable treatment setting from an observed value
whose standardized value happens to be close to zero.

The resulting treatment representation can be written as

\begin{equation}
\bm{a}_t
=
\left[
\bm{a}_t^{\mathrm{mod}},
\bm{a}_t^{\mathrm{sub}},
\tilde{f}_t^{\mathrm{FiO}_2},
\tilde{f}_t^{\mathrm{flow}},
\tilde{f}_t^{\mathrm{IPAP}},
\tilde{f}_t^{\mathrm{EPAP}},
r_t^{\mathrm{FiO}_2},
r_t^{\mathrm{flow}},
r_t^{\mathrm{IPAP}},
r_t^{\mathrm{EPAP}}
\right],
\end{equation}

where $\bm{a}_t^{\mathrm{mod}}$ and
$\bm{a}_t^{\mathrm{sub}}$ denote the one-hot modality and subtype
representations, respectively. Encounters containing an unknown respiratory support modality were excluded from model development and evaluation.

\paragraph{Trajectory construction.}
We constructed fixed-length trajectories consisting of a 6-hour observed history followed by a 6-hour prediction horizon. The 6-hour horizon was chosen to focus on short-term physiological changes during respiratory support, which
can provide clinically relevant information about evolving respiratory status and deterioration \citep{roca2019index}.

For an encounter with a selected starting index $k$, the observed physiological history was defined as

\begin{equation}
\bm{X}_{k:k+5}
=
\left\{
\bm{x}_{k},
\bm{x}_{k+1},
\ldots,
\bm{x}_{k+5}
\right\},
\end{equation}

with corresponding time-since-last-measurement sequence

\begin{equation}
\bm{L}_{k:k+5}
=
\left\{
\bm{\ell}_{k},
\bm{\ell}_{k+1},
\ldots,
\bm{\ell}_{k+5}
\right\}.
\end{equation}

The subsequent 6-hour physiological trajectory was

\begin{equation}
\bm{X}_{k+6:k+11}
=
\left\{
\bm{x}_{k+6},
\bm{x}_{k+7},
\ldots,
\bm{x}_{k+11}
\right\},
\end{equation}

and the respiratory support sequence over the prediction horizon was

\begin{equation}
\bm{A}_{k+6:k+11}
=
\left\{
\bm{a}_{k+6},
\bm{a}_{k+7},
\ldots,
\bm{a}_{k+11}
\right\}.
\end{equation}

Encounters containing fewer than 12 hourly observations were excluded because they could not provide both the complete history and prediction horizon. Encounters containing exactly 12 observations were used in full. For encounters with more than 12 observations, one valid contiguous 12-hour window was randomly selected from the available trajectory. Window selection was performed once during dataset construction, after which the resulting training, validation, and external cohorts were fixed and reused across all experiments.

Each resulting sample therefore consists of

\begin{equation}
\left(
\bm{X}_{k:k+5},
\bm{L}_{k:k+5},
\bm{A}_{k:k+5},
\bm{z}^{\mathrm{static}},
\bm{X}_{k+6:k+11},
\bm{A}_{k+6:k+11}
\right),
\end{equation}

where $\bm{z}^{\mathrm{static}}$ contains the static demographic, clinical, and
comorbidity variables. The observed history and static information are used to
construct the initial latent patient state, whereas the respiratory
interventions over the prediction horizon condition the subsequent
autoregressive physiological rollout.

\paragraph{External cohort preprocessing.}
External cohorts were transformed using exactly the same preprocessing parameters estimated from the internal training cohort. In particular, missing physiological observations were imputed using the internal training medians, physiological variables were normalized using the internal training means and standard deviations after imputation, and respiratory support variables were
encoded using the treatment vocabulary and continuous-setting normalization parameters learned from the internal training cohort. No preprocessing parameters were re-estimated on external data, which ensured that external evaluation reflected deployment of a fixed model and fixed preprocessing pipeline to previously unseen patient populations.

\section{Baselines}
\label{app:baselines}

We compare against a set of representative physiological trajectory and clinical world-models. GRU-based recurrent models are commonly used for longitudinal EHR dynamics, and recent clinical world models such as SepsisAgent \citep{wu2026agentifying} and medDreamer \citep{xu2026meddreamer} also adopt GRU-based transition architectures. We therefore include both a standard action-conditioned GRU baseline and adaptations of these recent clinical world models to the respiratory trajectory prediction setting.

To ensure a fair comparison, all learning-based methods use the same training and validation splits, preprocessing pipeline, physiological input variables, treatment representations, prediction horizons, and observation masks. Hyperparameters are tuned independently for each method using Optuna \citep{akiba2019optuna} with the same number of trials and the same validation objective. We search over learning rate, hidden dimension, number of layers, dropout, and weight decay, and select the configuration achieving the lowest validation RMSE. The same early-stopping criterion is applied across methods. For literature-derived world models, we preserve their core transition architectures and modify only the task-specific input and output interfaces required for respiratory trajectory prediction.

\begin{itemize}

\item \textbf{Persistence}: A non-parametric baseline that assumes the patient's physiological state remains unchanged over the prediction horizon. Specifically, the most recent observed state is carried forward as the prediction for all future time points.

\item \textbf{Action-GRU}: A standard action-conditioned recurrent dynamics baseline that encodes the observed physiological history with a GRU \citep{cho-etal-2014-learning} and autoregressively predicts future states conditioned on respiratory support interventions.

\item \textbf{SepsisAgent-WM \citep{wu2026agentifying}}: Following the Clinical World Model used in SepsisAgent, we use a GRU-based world model that encodes the full patient trajectory, including both physiological states and prior treatment actions, and predicts the next physiological state conditioned on the current respiratory support intervention.

\item \textbf{medDreamer-WM \citep{xu2026meddreamer}}: Following the latent world model in medDreamer, we model stochastic patient dynamics in a latent state space using GRU-based posterior and prior transition models. The posterior infers the latent state from observed physiology and treatment history, while the prior predicts future latent states conditioned on the previous latent state and respiratory support action. Future physiological states are reconstructed from the predicted latent trajectory.

\end{itemize}
\section{Additional Analysis of Response-Aware Dynamics}
\label{app:theory}

This section clarifies two design choices in the proposed model. We first
describe the role of the explicit response pathway in latent-state updating.
We then analyze the reference constraint used to separate baseline dynamics
from respiratory-support-associated deviations. The analysis concerns the
model parameterization and observational transition dynamics and does not
establish causal treatment effects.

\subsection{Explicit response representation}
\label{app:response_theory}

Let $X_t$ denote the physiological state at time $t$, and define the
short-term physiological response as
\begin{equation}
R_t = X_{t+1}-X_t.
\label{eq:app_response}
\end{equation}

In the proposed model, the response pathway is used to represent physiological
change explicitly during latent-state updating. Under free-running prediction,
the next physiological state is generated by
\begin{equation}
\hat{X}_{t+1}
=
\tilde{X}_t+\Delta\hat{X}_t,
\end{equation}
and therefore
\begin{equation}
\hat{R}_t
=
\hat{X}_{t+1}-\tilde{X}_t
=
\Delta\hat{X}_t.
\label{eq:app_pred_response}
\end{equation}

Thus, the model-generated response does not introduce an additional observation
during inference. Instead, the predicted change is passed through a separate
response encoder,
\begin{equation}
Z_t^R
=
E_{\mathrm{response}}(\hat{R}_t),
\end{equation}
and supplied to the state-update function together with the predicted
physiological level:
\begin{equation}
S_{t+1}
=
U_{\mathrm{state}}
\left(
S_t,
\hat{X}_{t+1},
Z_t^A,
Z_t^R
\right).
\end{equation}

The response pathway therefore provides an explicit representation of
physiological change rather than an additional source of information at
inference time. This parameterization allows the state-update module to process
the predicted physiological level and its change through separate learned
representations.

During the initial response-guided training period, the response encoder may
instead receive the observed change
\begin{equation}
R_t^{\mathrm{obs}}
=
X_{t+1}-X_t.
\end{equation}
This supervision is used only for the response input to the latent-state
update; the physiological trajectory itself remains autoregressive. The
probability of using the observed response is reduced to zero after the initial
training period, after which both training and evaluation use model-generated
responses.

\subsection{Reference-anchored decomposition of physiological dynamics}
\label{app:decomposition_theory}

We consider the decomposition of the conditional physiological transition
into baseline dynamics and a respiratory-support-associated deviation.

Let
\begin{equation}
F(s,a)
=
\mathbb{E}
\left[
\Delta X_t
\mid
S_t=s,A_t=a
\right]
\label{eq:conditional_transition}
\end{equation}
denote the conditional mean physiological change for latent state $s$ and
respiratory-support action $a$.

For the analysis, we write the model decomposition as
\begin{equation}
F(s,a)
=
f_{\mathrm{base}}(s)
+
d(s,a),
\label{eq:additive_decomposition}
\end{equation}
where $f_{\mathrm{base}}(s)$ denotes the shared state-dependent dynamics and
$d(s,a)$ denotes the effective respiratory-support-associated deviation. In
the implemented model, this deviation corresponds to
\begin{equation}
d(s,a)
=
g(s,a)
\odot
F_{\mathrm{treat}}
\left(
E_{\mathrm{action}}(s,a)
\right).
\label{eq:effective_treatment_component}
\end{equation}

Without a reference constraint, the decomposition in
Eq.~\ref{eq:additive_decomposition} is not unique. For any function $c(s)$,
\begin{align}
f_{\mathrm{base}}'(s)
&=
f_{\mathrm{base}}(s)+c(s), \\
d'(s,a)
&=
d(s,a)-c(s),
\end{align}
and therefore
\begin{equation}
f_{\mathrm{base}}'(s)+d'(s,a)
=
F(s,a).
\end{equation}

We use room air as the reference respiratory-support condition. Let
\begin{equation}
a_0=\text{room air},
\end{equation}
and consider the reference constraint
\begin{equation}
d(s,a_0)=0.
\label{eq:reference_constraint}
\end{equation}

\begin{proposition}[Reference-anchored decomposition]
\label{prop:anchored_decomposition}
Let $F(s,a)$ be a fixed conditional mean transition and let $a_0$ be a
reference action. Among additive decompositions satisfying
\begin{equation}
F(s,a)
=
f_{\mathrm{base}}(s)+d(s,a)
\end{equation}
and
\begin{equation}
d(s,a_0)=0,
\end{equation}
the decomposition is uniquely given by
\begin{align}
f_{\mathrm{base}}(s)
&=
F(s,a_0),
\label{eq:anchored_base}
\\
d(s,a)
&=
F(s,a)-F(s,a_0).
\label{eq:anchored_treatment}
\end{align}
\end{proposition}

\begin{proof}
Evaluating the additive decomposition at $a_0$ gives
\begin{equation}
F(s,a_0)
=
f_{\mathrm{base}}(s)+d(s,a_0).
\end{equation}
Using $d(s,a_0)=0$ yields
\begin{equation}
f_{\mathrm{base}}(s)=F(s,a_0).
\end{equation}
Substituting this expression back into
Eq.~\ref{eq:additive_decomposition} gives
\begin{equation}
d(s,a)
=
F(s,a)-F(s,a_0),
\end{equation}
which uniquely determines both components.
\end{proof}

Proposition~\ref{prop:anchored_decomposition} gives the decomposition a
reference-relative interpretation. Under the hard reference constraint, the
baseline component equals the conditional transition associated with room air,
and the support-associated component represents the deviation from that
reference transition.

The implemented model uses a soft rather than hard reference constraint.
Specifically, during room-air prediction steps we penalize
\begin{equation}
\mathcal{L}_{\mathrm{RA}}
=
\mathbb{E}
\left[
\left\|
d(S_t,A_t)
\right\|_2^2
\middle|
A_t=a_0
\right],
\end{equation}
which corresponds to
\begin{equation}
\mathcal{L}_{\mathrm{RA}}
=
\mathbb{E}
\left[
\left\|
\bm{g}_t
\odot
F_{\mathrm{treat}}(\bm{z}^{a}_t)
\right\|_2^2
\middle|
A_t=a_0
\right].
\end{equation}

The learned model is therefore encouraged toward the reference-anchored
decomposition but is not constrained to satisfy
$d(s,a_0)=0$ exactly. In particular, the proposition establishes uniqueness
only under the corresponding hard reference constraint.

Finally, the decomposition is defined in terms of observational conditional
transitions. The support-associated deviation should therefore not be
interpreted as an identified causal effect of respiratory support. Respiratory
support is not randomized in the retrospective data, and treatment assignment
may depend on measured and unmeasured patient characteristics.

\section{Additional Stratified Results}
\label{app:support_stratification}

Table~\ref{tab:support_stratification} reports the corresponding absolute RMSE
values for each respiratory-support subgroup.
\begin{table*}
\centering
\caption{
Trajectory prediction performance stratified by respiratory support modality.
RMSE is computed over all observed targets across the 1--6 h prediction horizon
within each subgroup; lower is better.
$N$ denotes the number of evaluation windows.
}
\label{tab:support_stratification}

\small
\setlength{\tabcolsep}{2.5pt}
\renewcommand{\arraystretch}{1.2}

\begin{tabular}{llcccccc}
\hline
\textbf{Cohort}
& \textbf{Support}
& \textbf{$N$}
& \textbf{Persistence}
& \textbf{Action-GRU}
& \textbf{SepsisAgent-WM}
& \textbf{MedDreamer-WM}
& \textbf{Ours} \\
\hline

\multirow{4}{*}{Val}
& Room air
& 491
& 0.546
& 0.517
& 0.517
& 0.510
& 0.515\\

& Conventional O$_2$
& 1,056
& 0.540
& 0.509
& 0.510
& 0.502
& 0.500 \\

& HFNC
& 314
& 0.527
& 0.504
& 0.502
& 0.510
& 0.512\\

& NIV
& 132
& 0.665
& 0.637
& 0.633
& 0.634
& 0.630 \\

\cline{1-8}

\multirow{4}{*}{Temporal}
& Room air
& 161
& 0.528
& 0.463
& 0.469 
& 0.459
& 0.456\\

& Conventional O$_2$
& 346
& 0.510
& 0.479
& 0.477
& 0.474
& 0.474 \\

& HFNC
& 105
& 0.819
& 0.800
& 0.800
& 0.794
& 0.790 \\

& NIV
& 30
& 0.639
& 0.607
& 0.602
& 0.593
& 0.604 \\

\cline{1-8}

\multirow{4}{*}{External}
& Room air
& 265
& 0.755
& 0.730
& 0.727
& 0.738
& 0.727 \\

& Conventional O$_2$
& 465
& 0.895
& 0.871
& 0.873
& 0.883
& 0.863 \\

& HFNC
& 100
& 1.000
& 0.983
& 0.977
& 0.984
& 0.985 \\

& NIV
& 33
& 1.437
& 1.433
& 1.427
& 1.452
& 1.425 \\

\hline
\end{tabular}
\end{table*}

\begin{table*}[t]
\centering
\caption{
Feature-level RMSE on physiologically meaningful changes
($|x_{t+h}-x_t|>0.25$ SD) across internal and external cohorts.
$N$ denotes the number of observed meaningful-change targets.
Lower RMSE indicates better performance.
}
\label{tab:feature_meaningful}

\small
\setlength{\tabcolsep}{4.2pt}
\renewcommand{\arraystretch}{1.08}

\begin{tabular}{lrrrrrr}
\hline
\textbf{Feature}
& \textbf{$N$}
& \textbf{Persistence}
& \textbf{Action-GRU}
& \textbf{SepsisAgent-WM}
& \textbf{MedDreamer-WM}
& \textbf{Ours} \\
\hline

\multicolumn{7}{l}{\textit{Val}} \\

O$_2$Sat
& 8,275
& 1.057
& 0.914
& 0.905
& 0.872
& 0.850\\

Resp
& 8,524
& 1.029
& 0.898
& 0.902
& 0.819
& 0.893\\

FiO$_2$
& 151
& 2.474
& 2.426
& 2.462
& 2.277
& 1.878 \\

PaO$_2$
& 272
& 1.440
& 1.422
& 1.416
& 1.460
& 1.396 \\

PaCO$_2$
& 278
& 1.652
& 1.592
& 1.617
& 1.664
& 1.518 \\

pH
& 377
& 1.614
& 1.498
& 1.569
& 1.606
& 1.563\\

HR
& 6,107
& 0.737
& 0.687
& 0.690
& 0.715
& 0.691 \\

MAP
& 7,161
& 1.012
& 0.903
& 0.896
& 0.858
& 0.857\\

SBP
& 6,966
& 0.911
& 0.825
& 0.825
& 0.808
& 0.816\\

DBP
& 7,345
& 1.075
& 0.944
& 0.939
& 0.899
& 0.894 \\

\cline{1-7}

\multicolumn{7}{l}{\textit{Temporal}} \\

O$_2$Sat
& 2,777
& 1.053 & 0.909 & 0.905&0.897  &0.802\\

Resp
& 2,727
& 1.023 & 0.882 & 0.883 & 0.782 & 0.774 \\

FiO$_2$
& 48
& 2.287 & 2.115 & 2.281 & 2.275 & 2.104 \\

PaO$_2$
& 120
& 1.474 & 1.377 & 1.377 & 1.302 & 1.347 \\

PaCO$_2$
& 104
& 1.134 & 1.044 & 1.100 & 1.126 & 1.053 \\

pH
& 145
& 2.049 & 1.986 & 2.047 & 1.938 & 2.030 \\

HR
& 1,969
& 0.711 & 0.676 & 0.670 &0.682 & 0.670 \\

MAP
& 2,211
& 0.996 & 0.885 & 0.882 & 0.838 & 0.837 \\

SBP
& 2,086
& 0.904 & 0.814 & 0.813 & 0.798 & 0.803 \\

DBP
& 2,206
& 1.058 & 0.925 & 0.923 & 0.871 & 0.828 \\

\cline{1-7}

\multicolumn{7}{l}{\textit{External}} \\

O$_2$Sat
& 3,649
& 1.194 & 1.034 & 1.032 & 1.011 & 1.010 \\

Resp
& 3,766
& 1.312 & 1.144 & 1.144& 1.052 & 1.139 \\

FiO$_2$
& 85
& 2.365 & 2.341 & 2.338 & 2.351 & 1.886 \\

PaO$_2$
& 58
& 1.729 & 1.611 & 1.642 & 1.609 & 1.566 \\

PaCO$_2$
& 67
& 1.380 & 1.475 & 1.381 & 1.403 & 1.429 \\

pH
& 82
& 2.145 & 2.164 & 2.101 & 2.178 & 2.047 \\

HR
& 2,607
& 0.774 & 0.722 & 0.725 & 0.747 & 0.709 \\

MAP
& 3,498
& 0.955 & 0.870 & 0.869 & 0.831 & 0.821 \\

SBP
& 3,415
& 1.011 & 0.934 & 0.932 & 0.918 & 0.910 \\

DBP
& 3,135
& 0.877 & 0.775 & 0.782 & 0.745 & 0.729 \\

\hline
\end{tabular}
\end{table*}

\end{document}